\documentclass[twoside]{article}
\usepackage[accepted]{aistats2019}
\usepackage{url}            
\usepackage{nicefrac}   
\usepackage{amsmath,amsthm,amssymb,amsfonts} 
\usepackage{mathtools}
\mathtoolsset{showonlyrefs}
\usepackage[round]{natbib}

\mathtoolsset{showonlyrefs}

\renewcommand{\Pr}{\field{P}}

\newcommand{\bta}{\boldsymbol{\eta}}

\newcommand{\bg}{\boldsymbol{g}}
\newcommand{\bx}{\boldsymbol{x}}

\newcommand{\bu}{\boldsymbol{u}}
\newcommand{\by}{\boldsymbol{y}}

\newcommand{\field}[1]{\mathbb{#1}}

\newcommand{\R}{\field{R}}

\newcommand{\E}{\field{E}}

\newtheorem{lemma}{Lemma}
\newtheorem{theorem}{Theorem}

\newtheorem{example}{Example}

\begin{document}

\runningtitle{On the Convergence of Stochastic Gradient Descent with Adaptive Stepsizes}

\twocolumn[

\aistatstitle{On the Convergence of Stochastic Gradient Descent\\with Adaptive Stepsizes}

\aistatsauthor{ Xiaoyu Li \And Francesco Orabona}

\aistatsaddress{ Boston University \And  Boston University } ]

\begin{abstract}
Stochastic gradient descent is the method of choice for large scale optimization of machine learning objective functions. Yet, its performance is greatly variable and heavily depends on the choice of the stepsizes. This has motivated a large body of research on adaptive stepsizes. However, there is currently a gap in our theoretical understanding of these methods, especially in the non-convex setting. In this paper, we start closing this gap: we theoretically analyze in the convex and non-convex settings a generalized version of the AdaGrad stepsizes. We show sufficient conditions for these stepsizes to achieve almost sure asymptotic convergence of the gradients to zero, proving the first guarantee for generalized AdaGrad stepsizes in the non-convex setting.
Moreover, we show that these stepsizes allow to automatically adapt to the level of noise of the stochastic gradients in both the convex and non-convex settings, interpolating between $O(1/T)$ and $O(1/\sqrt{T})$, up to logarithmic terms.
\end{abstract}
 
\section{INTRODUCTION}
In recent years, Stochastic Gradient Descent (SGD) has become the tool of choice to train machine learning models. In particular, in the Deep Learning community, it is widely used to minimize the training error of deep networks. In this setting, the stochasticity arises from the use of so-called \emph{mini-batches}, that allows to keep the complexity per iteration constant with respect to the size of the training set.

More in details, SGD iteratively updates the solution as $\bx_{t+1}=\bx_t-\eta_t \bg(\bx_t,\xi_t)$, starting from an arbitrary point $\bx_1$, where $\bg(\bx_t,\xi_t)$ is a stochastic gradient in $\bx_t$ that depends on the stochastic variable $\xi_t$.Classic convergence analysis of the SGD algorithm for non-convex smooth functions relies on conditions on the positive stepsizes $\eta_t$~\citep{RobbinsM51}. In particular, sufficient conditions are that $(\eta_t)_{t=1}^\infty$ is a deterministic sequence of non-negative numbers that satisfies
\begin{align}
\label{eq:conditions_stepsize}
\sum_{t=1}^\infty \eta_t = \infty \quad\text{ and }\quad \sum_{t=1}^\infty \eta_t^2 < \infty.
\end{align}

However, state-of-the-art SGD variants use \emph{adaptive stepsizes}, that is $\eta_t$ is a function of past stochastic gradients.
These stepsizes are believed to require less tweaking to achieve good performance in machine learning applications and we have some partial explanations in the convex setting, i.e. sparsity of the gradients~\citep{DuchiHS11}. However, in the non-convex setting, we do not have any theory explaining the better performance.

Indeed, for a large number of SGD variants employed by practitioners the conditions above are not satisfied. In fact, these algorithms are often designed and analyzed for the convex domain under restrictive conditions, e.g. bounded domains, or they do not provide convergence guarantees at all, \citep[e.g.][]{Zeiler12}, or even worse they are known to fail to converge on simple one-dimensional convex stochastic optimization problems~\citep{ReddiKK18}. Even considering an \emph{infinite} number of iterations, the behavior of these algorithms is often unknown.

We focus on a generalized version of the adaptive stepsizes popularized by AdaGrad~\citep{DuchiHS11}. This kind of stepsizes has become the basis of all other adaptive optimization algorithms used in machine learning, \citep[e.g.][]{Zeiler12,TielemanH12,KingmaB15,ReddiKK18}.
We analyze two types of step size: a global step size
\begin{equation}
\label{eq:eta}
\eta_t=\frac{\alpha}{\left(\beta+ \sum_{i=1}^{t-1} \|\bg(\bx_i,\xi_i)\|^2\right)^{\nicefrac{1}{2}+\epsilon}}
\end{equation}
and a coordinate-wise one
\begin{equation}
\label{eq:eta2}
\eta_{t,j}=\frac{\alpha}{\left(\beta+ \sum_{i=1}^{t-1} g(\bx_i,\xi_i)_j^2\right)^{\nicefrac{1}{2}+\epsilon}}, j=1, \cdots,d
\end{equation}
where $\alpha>0$ and $\beta, \epsilon\geq0$. Note that, with $\epsilon=0$, \eqref{eq:eta2} are the coordinate-wise stepsizes used in AdaGrad~\citep{DuchiHS11}, while \eqref{eq:eta} have been used in online convex optimization to achieve adaptive regret guarantees, \citep[e.g.][]{RakhlinS13,OrabonaP18}. The additional parameter $\epsilon$ allows us to increase the decrease rate of the stepsize and it will be critical to obtain our almost sure convergence results.

In this paper, we want to answer two basic questions: 1) Are there conditions under which the generalized AdaGrad stepsize converge almost surely with an infinite number of iterations in the non-convex setting? 2) Are there conditions under which the rate is better than the one of the plain SGD with decreasing stepsizes?

We answer positively to both questions. More in details, the contributions of this paper are the following:
\begin{itemize}
\item In Section~\ref{sec:almost_sure}, we prove for the \emph{first} time in the non-convex setting almost sure asymptotic convergence to zero of the gradients of SGD with both coordinate-wise and global adaptive stepsizes.
\item In Section~\ref{sec:convex}, we prove that in the convex setting the generalized global AdaGrad stepsizes adapts to the noise level, through a finite-time convergence rate. In particular, we show that, depending on the noise level, SGD with the generalized AdaGrad updates automatically interpolates between the convergence rates of Gradient Descent (GD) and SGD, up to polylogarithmic terms. We do so \emph{removing the strong assumptions} present in previous analyses.
\item In Section~\ref{sec:adapt}, similarly to the results of Section~\ref{sec:convex}, we show that in the non-convex setting the generalized global AdaGrad stepsizes adapts to the noise level, through a novel finite-time convergence rate. A low noise will result in an automatic faster convergence. As far as we know, these are the \emph{first} theoretical results for the advantage of AdaGrad-like stepsizes over the plain SGD in the non-convex setting.
\end{itemize}

The next Section discusses more in details the related work, while Section~\ref{sec:def} introduces formally the setting, and Section~\ref{sec:stepsize} discusses the details of the adaptive stepsizes considered in this work.

\section{RELATED WORK}

In the convex setting, adaptive stepsizes have a long history. They were first proposed in the online learning literature~\citep{AuerCG02} and adopted into the stochastic optimization one later~\citep{DuchiHS11}. In particular, in \citep{DuchiHS11} they prove that AdaGrad can converge faster if the gradients are sparse and the function is convex. Yet, most of these studies assumed the optimization to be constrained in a convex bounded set. This assumption is often false in many applications of optimization for machine learning. \citet{YousefianNS12} analyze different adaptive stepsizes, but only for strongly convex optimization. 
Recently, \citet{WuWB18} have analyzed a choice of adaptive stepsizes similar to the global stepsizes we consider, but their result in the convex setting requires the norm of the gradients strictly greater than zero. \citet{LevyYC18} propose an acceleration method with adaptive stepsizes which are also similar to our global ones, proving the $\tilde{O}(1/T^2 )$ convergence in the deterministic smooth case and $\tilde{O} (1/\sqrt{T})$ in both general deterministic case and stochastic smooth case, but requiring a bounded-domain assumption. 

The convergence of a random iterate of SGD for non-convex smooth functions has been proved by \citet{Ghadimi13}, and it was already implied by the results in \citet{Bottou91}. With additional regularity assumptions, these results imply almost sure convergence of the gradient to zero~\citep{Bottou91,Bottou16}. In alternative to the regularity assumptions, \citet{Bottou98} proposed to assume that beyond a certain horizon the update always moves the iterate closer to the origin on average, that implies the confinement in a bounded domain and, in turn, the almost sure convergence.
On the other hand, the weakest assumptions for the almost sure convergence of SGD for non-convex smooth functions have been established in \citet{BertsekasT00}: the variance of the noise on the gradient in $\bx_t$ can grow as $1+\|\nabla f(\bx_t)\|^2$, $f$ is lower bounded, and the stepsizes satisfy \eqref{eq:conditions_stepsize}.
However, both approaches do not cover adaptive stepsizes.

The first work we know on adaptive stepsizes for non-convex stochastic optimization is \citet{KresojaLS17}. They study the convergence of a choice of adaptive stepsizes that require access to the function values, under strict conditions on the direction of the gradients. \citet{WuWB18} also consider adaptive stepsizes, but they only consider deterministic gradients in the non-convex setting. Later, \citet{WardWB18}, independently and the same time with us, improved their guarantees proving results similar to our Theorems~\ref{thm:convex} and \ref{thm:sgd_adaptive}. They use the original AdaGrad stepsizes, but with the assumption of bounded expected squared norm of the stochastic gradients. Some other related works were proposed after our submission. \citet{ZhouTYCG18} analyze an adaptive gradient method in the non-convex setting, but their bounds give advantages only in very sparse case.

A weak condition for almost sure convergence to the global optimum of non-convex functions was proposed in \citet{Bottou98} and recently independently reproposed in~\citet{ZhouMBBG17}. However, this condition implies the very strong assumption that the gradients never point in the opposite direction of the global optimum. In this paper, in our most restrictive case in Section~\ref{sec:almost_sure}, we will only assume the function to be smooth and Lipschitz.

\section{PROBLEM SET-UP}
\label{sec:def}

\paragraph{Notation.} We denote vectors and matrices by bold letters, e.g. $\bx \in \R^d$. The coordinate $j$ of a vector $\bx$ is denoted by $x_j$ and as $(\nabla f(\bx))_j$ for the gradient $\nabla f(\bx)$. We denote by $\E[\cdot]$ the expectation with respect to the underlying probability space and by $\E_t[\cdot]$ the conditional expectation with respect to the past, that is, with respect to $\xi_1, \cdots, \xi_{t-1}$. We use L2 norms.

\paragraph{Setting and Assumptions.}
We consider the following optimization problem
$\min_{ \bx \in \R^d }  \  f(\bx)$,
where $f(\bx):\R^d\rightarrow \R$ is a function bounded from below.
We will make different assumptions on the objective function $f$, depending on the setting. In particular, we will always assume that
\begin{itemize}
\item[(\textbf{H1})] $f$ is \emph{$M$-smooth}, that is, $f$ is differentiable and its gradient is $M$-Lipschitz, i.e. $\|\nabla f(\bx)- \nabla f(\by)\| \leq M \|\bx-\by\|, \ \forall \bx, \by \in \R^d$.
\end{itemize}
Note that (\textbf{H1}), for all $\bx,\by \in \R^d$, implies~\citep[Lemma 1.2.3]{Nesterov2003}
\begin{equation}
\label{eq:smooth2}
\left|f(\by)-f(\bx)-\langle \nabla f(\bx), \by-\bx\rangle\right|
\leq \frac{M}{2}\|\by-\bx\|^2~.
\end{equation}
Sometimes, we will also assume that
\begin{itemize}
\item[(\textbf{H2})] $f$ is \emph{$L$-Lipschitz}, i.e. $|f(\bx)-f(\by)|\leq L\|\bx-\by\|, \ \forall \bx, \by \in \R^d$. 
\end{itemize}

We assume that we have access to a stochastic first-order black-box oracle, that returns a noisy estimate of the gradient of $f$ at any point $\bx \in \R^d$. That is, we will use the following assumption
\begin{itemize}
\item[(\textbf{H3})] We receive a vector $\bg(\bx,\xi)$ such that $\E_\xi [\bg(\bx,\xi)]=\nabla f(\bx)$ for any $\bx \in \R^d$.
\end{itemize}
We will also make alternatively one of the following assumptions on the variance of the noise.
\begin{itemize}
\item[(\textbf{H4})] The noise in the stochastic gradient has bounded support, that is $\|\bg(\bx,\xi)-\nabla f(\bx)\|\leq S, \ \forall \bx$.
\item[(\textbf{H4'})] The stochastic gradient satisfies $\E_\xi \left[\exp\left(\|\nabla f(\bx) - g(\bx,\xi)\|^2/\sigma^2\right)\right]\leq \exp(1), \  \forall \bx$.
\end{itemize}

(\textbf{H4'}) has been already used by \cite{NemirovskiJLS09} to prove high probability convergence guarantees.
This condition allows to control the expectation of the maximum of the terms $\|\nabla f(\bx_t) - g(\bx_t,\xi_t)\|^2$.
Note that, using Jensen's inequality, this condition implies a bounded variance.
Also, (\textbf{H4}) implies (\textbf{H4'}).

\section{KEEPING THE UPDATE DIRECTION UNBIASED}
\label{sec:stepsize}

A key difference between the generalized AdaGrad stepsizes in \eqref{eq:eta} and \eqref{eq:eta2} with the AdaGrad stepsizes in \citet{DuchiHS11} is the fact that $\bg(\bx_t,\xi_t)$ is not used in $\eta_t$. It is easy to see that doing otherwise introduces a spurious bias in the update direction.
Indeed, as we show in the Example below, if the stepsize does depend on the current gradient, things can go wrong. The details can be found in the Appendix. 
\begin{example}
\label{ex:ninety}
There exist a convex differentiable function satisfying (\textbf{H1}), an additive noise on the gradients satisfying (\textbf{H4}), and a sequence of gradients such that for a given $t$ we have $\E_{\xi_t}[\langle\eta_{t+1}\bg(\bx_t,\xi_t),\nabla f(\bx_t)\rangle]<0$. 
\end{example}
In words, the example says that including the current noisy gradient in $\eta_t$ (that is, using $\eta_{t+1}$) can make the algorithm deviate in expectation more than $90$ degrees from the correct direction. While in the convex bounded case the algorithm can recover, it is intuitive that this could have catastrophic consequences in the unconstrained non-convex setting, especially when the function is not Lipschitz. So, in the following, we will analyze this minor variant of the AdaGrad stepsizes.

On the other hand, this difference makes the analysis more involved, because the quantity $\sum_{t=1}^T \eta_t^2 \|\bg(\bx_t,\xi_t)\|^2$ cannot be bounded anymore in a straightforward way, see Lemma~\ref{lemma:bounded_sum_squares} in the next Section. Previous analyses, \citep[e.g.][]{DuchiHS11}, solved this issue by assuming the knowledge of the Lipschitz constant of the function $f$, while we will assume the function to be Lipschitz only to prove the asymptotic guarantee and no knowledge of it. We believe that removing the assumption of knowing the Lipschitz constant more closely follow the use in real-world applications.

In the following, we will show that these stepsize allow to prove adaptive guarantees in the convex and non-convex settings.

\section{ALMOST SURE CONVERGENCE FOR NON-CONVEX FUNCTIONS}
\label{sec:almost_sure}

In this section, we show that SGD with the generalized AdaGrad stepsizes in \eqref{eq:eta} and \eqref{eq:eta2} allows to decrease the gradients to zero almost surely, that is, with probability 1. This is considered a required basic property for any optimization algorithm.

The stepsizes in \eqref{eq:eta} and \eqref{eq:eta2} \emph{do not satisfy} \eqref{eq:conditions_stepsize}, not even in expectation, because the $\bg(\bx_t,\xi_t)$ could decrease fast enough to have $\sum_{t=1}^\infty \eta_t^2=\infty$. Hence, the results here cannot be obtained from the classic results in stochastic approximation~\citep[e.g.][]{BertsekasT00}.

Here, we will have to assume our strongest assumptions. In particular, we will need the function to be Lipschitz and the noise to have bounded support. This is mainly needed in order to be sure that the sum of the stepsizes diverges.

We now state our almost sure convergence results.
\begin{theorem}
\label{thm:convergence_sgd}
Assume (\textbf{H1}, \textbf{H2}, \textbf{H3}, \textbf{H4}).
The stepsizes are chosen as in \eqref{eq:eta}, where $\alpha,\beta>0$ and $\epsilon \in (0,\frac{1}{2}]$.
Then, the gradients of SGD converges to zero almost surely.
Moreover, $\lim\inf_{t\rightarrow \infty} \|\nabla f(\bx_t)\|^2 t^{\nicefrac12-\epsilon}=0$ almost surely.
\end{theorem}

We also state a similar result for the coordinate-wise stepsizes in \eqref{eq:eta2}. We remind the reader that these stepsizes closely mirror the ones used in AdaGrad, but with the power of the denominator $\frac{1}{2}+\epsilon$ with $\epsilon>0$, rather than $\frac{1}{2}$. Also, differently from what is stated in the original AdaGrad paper, here we do not project onto a bounded closed convex set. This mirrors the actual implementation of AdaGrad in machine learning libraries, e.g. Tensorflow~\citep{Tensorflow15}.
\begin{theorem}
\label{thm:convergence_adagrad}
Assume (\textbf{H1}, \textbf{H2}, \textbf{H3}, \textbf{H4}).
The stepsizes are given by a diagonal matrix $\bta_t$ whose diagonal values are defined in \eqref{eq:eta2}, where $\alpha,\beta>0$ and $\epsilon \in (0,\frac{1}{2}]$.
Then, the gradients of SGD converges to zero almost surely. Moreover, $\lim\inf_{t\rightarrow \infty} \|\nabla f(\bx_t)\|^2 t^{\nicefrac12-\epsilon}=0$ almost surely.
\end{theorem}

As far as we know, the above theorems are the first results on the almost sure convergence of the gradients using generalized AdaGrad stepsizes and assuming $\epsilon>0$. In particular, Theorem~\ref{thm:convergence_adagrad} is the first theoretical support to the common heuristic of selecting the last iterate, rather than the minimum over the iterations.

For the proofs, we will need some technical lemmas, whose proofs are in the Appendix.
\begin{lemma}{\citep[Proposition~2]{Alber98}\citep[Lemma A.5]{Mairal13}}
\label{lemma:remove_liminf}
Let $(a_t)_{t \geq 1}, (b_t)_{t \geq 1}$ be two non-negative real sequences. Assume that $\sum_{t=1}^{\infty} a_t b_t$ converges and $\sum_{t=1}^{\infty} a_t$ diverges, and there exists $K \geq 0$ such that $|b_{t+1}-b_t| \leq K a_t$. Then $b_t$ converges to 0.
\end{lemma}

\begin{lemma}
\label{lemma:sum_bounded}
Let $a_0>0$, $a_i\geq 0, \ i=1,\cdots,T$ and $\beta>1$.
Then
$
\sum_{t=1}^T \frac{a_t}{(a_0+\sum_{i=1}^{t} a_i)^\beta} 
\leq \frac{1}{(\beta-1)a_0^{\beta-1}}
$.
\end{lemma}

We now state a Lemma that allows us to study the progress made in $T$ steps. The proof is in the Appendix.
\begin{lemma}
\label{lemma:basic_lemma}
Assume (\textbf{H1}, \textbf{H3}). Then, the iterates of SGD with stepsizes $\bta_t \in \R^{d \times d}$ satisfy the following inequality
{\small
\begin{align}
\E\left[\sum_{t=1}^T \langle \nabla f(\bx_t), \bta_t \nabla f(\bx_t)\rangle\right]
&\leq f(\bx_1)- f^* \\
&\quad + \frac{M}{2} \E\left[\sum_{t=1}^T \|\bta_t \bg(\bx_t,\xi_t)\|^2\right].
\end{align}
}
\end{lemma}

We can prove Theorem~\ref{thm:convergence_sgd}. Given that the proof of Theorem~\ref{thm:convergence_adagrad} is virtually identical to the one of Theorem~\ref{thm:convergence_sgd}, we defer its proof to the Appendix.
\begin{proof}[Proof of Theorem~\ref{thm:convergence_sgd}]
From the result in Lemma~\ref{lemma:basic_lemma}, taking the limit for $T\rightarrow \infty$ and exchanging the expectation and the limits because the terms are non-negative, we have
{\small
\[
\E\left[\sum_{t=1}^\infty \eta_t \| \nabla f(\bx_t)\|^2\right] 
\leq f(\bx_1)- f^\star+ \frac{M}{2}\E\left[\sum_{t=1}^\infty \|\eta_t \bg(\bx_t,\xi_t)\|_{2}^2\right].
\]
}
Observe that
{\small
\begin{align}
& \sum_{t=1}^{\infty}  \|\eta_t \bg(\bx_t,\xi_t)\|^2 \\
& =\sum_{t=1}^{\infty}  \eta_{t+1}^2 \| \bg(\bx_t,\xi_t)\|^2 + \sum_{t=1}^{\infty}  (\eta_t^2-\eta_{t+1}^2) \| \bg(\bx_t,\xi_t)\|^2 \\
& \leq \frac{\alpha^2}{2\epsilon \beta^{2\epsilon}} + \max_{t\geq 1} \|\bg(\bx_t,\xi_t)\|^2 \sum_{t=1}^{\infty} (\eta_t^2-\eta_{t+1}^2)  \\
& \leq \frac{\alpha^2}{2\epsilon \beta^{2\epsilon}} + \max_{t\geq 1} \|\bg(\bx_t,\xi_t)\|^2 \eta_1^2  \\
& \leq \frac{\alpha^2}{2\epsilon \beta^{2\epsilon}} + 2\eta_1^2\max_{t\geq 1} \|\nabla f(\bx_t)\|^2 + \|\nabla f(\bx_t)-\bg(\bx_t,\xi_t)\|^2  \\
& \leq \frac{\alpha^2}{2\epsilon \beta^{2\epsilon}} + 2 \frac{\alpha^2}{\beta^{1+2\epsilon}} (L^2+S^2) < \infty, \label{eq:convergence_sgd_eq1}
\end{align}
}
where in the first inequality we have used Lemma~\ref{lemma:sum_bounded}, and in the third one the elementary inequality $\|\bx+\by\|^2 \leq 2\|\bx\|^2 +2 \|\by\|^2$.

Hence, we have $\E\left[\sum_{t=1}^{\infty} \eta_t \| \nabla f(\bx_t)\|^2\right] < \infty$.
Now, note that $\E[X]<\infty$, where $X$ is a non-negative random variable, implies that $X<\infty$ with probability 1. In fact, otherwise $\Pr[X=\infty]>0$ implies $\E[X] \geq \int_{X=\infty} x d\Pr(X) = \infty$, contradicting our assumption. 
Hence, with probability 1, we have $\sum_{t=1}^{\infty} \eta_t \| \nabla f(\bx_t)\|^2 < \infty$.

Now, observe that the Lipschitzness of $f$ and the bounded support of the noise on the gradients gives
{\small
\begin{align}
\sum_{t=1}^{\infty} \eta_{t}
& = \sum_{t=1}^ {\infty} \frac{\alpha}{(\beta+\sum_{i=1}^{t-1} \|g(\bx_i,\xi_i)\|^2)^{\nicefrac12+\epsilon}} \\
& \geq  \sum_{t=1}^{\infty} \frac{\alpha}{(\beta+2(t-1)(L^2+S^2))^{\nicefrac12+\epsilon}}
= \infty.
\end{align}
}

Using the fact the $f$ is $L$-Lipschitz and $M$-smooth, we have
{\small
\begin{align}
& \left| \|\nabla f(\bx_{t+1})\|^2- \|\nabla f(\bx_t)\|^2\right|  \\
&  \quad = ( \|\nabla f(\bx_{t+1})\|+ \|\nabla f(\bx_t)\|) \cdot \left| \|\nabla f(\bx_{t+1})\|- \|\nabla f(\bx_t)\| \right| \\
&  \quad \leq 2L M \|\bx_{t+1}-\bx_t\|
= 2LM \|\eta_t \bg(\bx_t,\xi_t)\|  \\
& \quad \leq 2LM (L+S) \eta_t.
\end{align}
}
Hence, we can use Lemma~\ref{lemma:remove_liminf} to obtain $\lim_{t \to \infty} \|\nabla f(\bx_t)\|^2 = 0$.

For the second statement, observe that, with probability 1,
{\small
\begin{align}
 \sum_{t=1}^\infty &\|\nabla f(\bx_t)\|^2 t^{\nicefrac12-\epsilon} \frac{\alpha}{t(2L^2+2S^2+\beta)^{\nicefrac12+\epsilon}} \\
& \leq \sum_{t=1}^\infty \eta_t \|\nabla f(\bx_t)\|^2 <\infty,
\end{align}
}
where in the first inequality we used the Lipschitzness of $f$ and the bounded support of the noise on the gradients.
Hence, noting that $\sum_{t=1}^\infty \frac{1}{t} =\infty$, we have that $\lim\inf_{t\rightarrow \infty} \|\nabla f(\bx_t)\|^2 t^{\nicefrac12-\epsilon}=0$.
\end{proof}

Even if the above results hold with probability 1, the above convergence guarantees rates are only asymptotic. So, in the next Section, we show finite-time convergence rates in expectation. Moreover, we will show that the generalized AdaGrad stepsizes adapt to the level of noise for both the convex and non-convex case.

\section{ADAPTIVE CONVERGENCE RATES}

We will now show that the global generalized AdaGrad stepsizes give rises to adaptive convergence rates. In particular, we will show that for a large range of the parameters $\alpha, \beta, \epsilon$ and independently from the noise variance $\sigma$, the algorithms will have a faster convergence when $\sigma$ is small and worst-case optimal convergence when $\sigma$ is large. Note that to achieve the same behavior with SGD we should use a different stepsize for each level of noise. 

In the following, we will consider the convex and non-convex case.

\subsection{Adaptive Convergence for Convex Functions}
\label{sec:convex}
As a warm-up, in this section, we show that the global stepsizes \eqref{eq:eta} give adaptive rates of convergence that interpolate between the rate of GD and SGD, for a wide range of the parameters $\alpha, \beta$, and $\epsilon$ and without knowledge of the variance of the noise.
Note that, differently from the other proofs on SGD with adaptive rates~\citep[e.g.][]{DuchiHS11}, we do not assume to use projections onto bounded domains. This makes our novel proof more technically challenging, but at the same time, it mirrors the setting of many applications of SGD in machine learning optimization problems.

\begin{theorem}
\label{thm:convex}
Assume (\textbf{H1}, \textbf{H3}, \textbf{H4'}) and $f$ convex. Let the stepsizes set as in \eqref{eq:eta}, where $\alpha,\beta>0$, $0\leq \epsilon<\frac{1}{2}$, and $4\alpha M<\beta^{\nicefrac{1}{2}+\epsilon}$.
Then, the iterates of SGD satisfy the following bound
\begin{equation}
\small
\begin{aligned}
& \E \left[ \left(  f(\bar{\bx}_T)- f(\bx^\star) \right)^{\nicefrac{1}{2}-\epsilon}  \right] \\
& \leq  \frac{1}{T^{\nicefrac{1}{2}-\epsilon}} 
\max \left( 2^\frac{1}{\nicefrac{1}{2}-\epsilon} M^{\nicefrac{1}{2}+\epsilon} \gamma, 
 \left( \beta+T\sigma^2  \right)^{\nicefrac{1}{4}-\epsilon^2} \gamma^{\nicefrac{1}{2}-\epsilon}  \right),
\end{aligned}
\end{equation}

where $\bar{\bx}_T=\frac1T \sum_{t=1}^T \bx_t$ and 
\[
\gamma = 
\begin{cases} 
O\left(\frac{1+\alpha^2\ln T}{\alpha(1-\frac{4\alpha M}{\sqrt{\beta}})}\right), & \text{for } \epsilon=0\\ 
O\left(\frac{1+\alpha^2(\frac{1}{\epsilon}+\sigma^2 \ln T)}{\alpha(1-\frac{4\alpha M}{\beta^{\nicefrac{1}{2}+\epsilon}})}\right), & \text{for } \epsilon>0.
\end{cases}
\]
\end{theorem}

\paragraph{Remark.}
Using Markov's inequality, from the above bound it is immediate to get that, with probability at least $1-\delta$, we have
\begin{equation}
\small
\begin{aligned}
f&(\bar{\bx}_T)- f(\bx^\star)  \\
& \leq \frac{1}{\delta^{\frac{1}{\nicefrac{1}{2}-\epsilon}} T} \max \left(M^{\frac{\nicefrac{1}{2}+\epsilon}{\nicefrac{1}{2}-\epsilon}} \gamma^{\frac{1}{\nicefrac{1}{2}-\epsilon}}, (\beta + T \sigma^2)^{\nicefrac{1}{2}+\epsilon} \gamma \right).
\end{aligned}
\end{equation}
Up to polylog terms, if $\sigma=0$ this recovers the GD rate, $O(\tfrac{1}{T})$, and otherwise we get the worst-case optimal rate of SGD, $O(\tfrac{1}{\sqrt{T}})$. The same behavior was proved in \citet{DekelGBSX12} with the knowledge of $\sigma$ and stepsize depending on it. Instead, here we do not need to know the noise level nor assuming a bounded domain.
In the case the constants of the slow term are small compared with the ones of the first term, we can expect a first quick convergent phase, followed by a slow one, as it is often observed in empirical experiments.

For the proof, we first state some technical lemmas, whose proofs are in the Appendix.
\begin{lemma}
\label{lemma:smooth}
Assume (\textbf{H1}). Then $\|\nabla f(\bx)\|^2 \leq 2 M (f(\bx)- \min_{\by} f(\by)), \ \forall \bx$.
\end{lemma}
\begin{lemma}
\label{lemma:solvex}
If $x \geq 0$ and $x \leq C(A+Bx)^{\frac{1}{2}+\epsilon}$, then 
$
x < \max ( [ C (2B)^{\frac{1}{2}+\epsilon} ]^{\frac{1}{1/2-\epsilon}}, C(2A)^{\frac{1}{2}+\epsilon} )
$.
\end{lemma}

\begin{lemma}
\label{lemma:logsolvex}
If $x \geq 0$, $A, C, D \geq 0$, $B>0$, and $x^2 \leq (A+Bx)(C+D\ln (A+Bx))$, then $x < 32 B^3 D^2 + 2 B C + 8 B^2 D \sqrt{C} + A/B$.
\end{lemma}

\begin{lemma}
\label{lemma: exponential}
If $x,y \geq 0$ and $0 \leq p \leq 1$, then $(x+y)^p \leq x^p+y^p$.
\end{lemma}

\begin{lemma}
\label{lemma:bounded_sum_squares}
Assume (\textbf{H1}, \textbf{H3}, \textbf{H4'}). The stepsizes are chosen as \eqref{eq:eta}, where $\alpha,\beta,\epsilon\geq0$.
Then, 
\begin{equation}
\small
\begin{aligned}
\E\left[\sum_{t=1}^T \eta_t^2 \|\bg(\bx_t,\xi_t)\|^2\right] 
& \leq K + \frac{4\alpha^2}{\beta^{1+2\epsilon}} (1+\ln T) \sigma^2 \\ 
& \quad + \frac{4\alpha}{\beta^{\nicefrac{1}{2}+\epsilon}} \E\left[\sum_{t=1}^T \eta_t \|\nabla f(\bx_t)\|^2\right],\\
\end{aligned}
\end{equation}
where in the case of $\epsilon >0$, $ \small{K = \frac{\alpha^2}{2\epsilon \beta^{2\epsilon}}}$; when $\epsilon = 0$, $ \small{K = 2 \alpha^2  \ln \left( \sqrt{\beta + 2T\sigma^2 }+\sqrt{2} \E \left[ \sqrt{ \sum_{t=1}^T \|\nabla f(\bx_t)\|^2}  \right] \right)}$.

\end{lemma}

We can now prove the theorem.
\begin{proof}[Proof of Theorem~\ref{thm:convex}]
For simplicity, denote by $\delta_t:=f(\bx_t)-f(\bx^\star)$ and by $\Delta := \sum_{t=1}^T \delta_t$.

From the update of SGD we have that
\begin{equation}
\small
\begin{aligned}
\|\bx_{t+1}-\bx^\star\|^2 - \|\bx_{t}-\bx^\star\|^2 &=-2 \eta_t \langle \bg(\bx_t,\xi_t), \bx_t- \bx^\star\rangle\\
&\quad + \eta_t^2 \|\bg(\bx_t,\xi_t)\|^2.
\end{aligned}
\end{equation}

Taking the conditional expectation with respect to $\xi_1, \cdots, \xi_{t-1}$, we have that
\begin{equation}
\small
E_t[\langle \bg(\bx_t,\xi_t), \bx_t-\bx^\star \rangle]
=  \langle \nabla f(\bx_t), \bx_t-\bx^\star \rangle
\geq \delta_t,
\end{equation}
where in the inequality we used the fact that $f$ is convex.
Hence, summing over $t=1$ to $T$, we have
\begin{equation}
\small
\begin{aligned}
\E\left[\sum_{t=1}^T \eta_t \delta_t\right] 
\leq \frac{1}{2}\|\bx^\star-\bx_1\|^2 + \frac{1}{2}\E\left[\sum_{t=1}^T \eta_t^2 \|\bg(\bx_t,\xi_t)\|^2 \right].
\end{aligned}
\end{equation}

From Lemma~\ref{lemma:smooth} and Lemma~\ref{lemma:bounded_sum_squares}, when $\epsilon>0$ we have that
\begin{equation}
\small
\label{eq:convex_eq1}
\begin{aligned}
\left(1- \frac{4\alpha M}{\beta^{\nicefrac{1}{2}+\epsilon}}\right)\E\left[\sum_{t=1}^T \eta_t \delta_t\right] 
& \leq \frac{1}{2}\|\bx^\star-\bx_1\|^2 + \frac{\alpha^2}{4\epsilon \beta^{2\epsilon}} \\
& \quad + \frac{2\alpha^2}{\beta^{1+2\epsilon}} (1+\ln T) \sigma^2 .
\end{aligned}
\end{equation} 
On the other hand, when $\epsilon=0$ we have 
\begin{equation}
\small
\label{eq:convex_eq2_epsl0}
\begin{aligned}
& \left(1- \frac{4\alpha M}{\beta^{\nicefrac{1}{2}}}\right)\E\left[\sum_{t=1}^T \eta_t \delta_t\right] \\
&\quad \leq  \frac{1}{2} \| \bx_1 - \bx^{\star} \|^2 + \frac{2 \alpha^2}{\beta} (1+\ln T)\sigma^2 \\
&\qquad +\alpha^2 \ln \left(  \sqrt{\beta + 2T\sigma^2} + 2\sqrt{M} \E \left[ \sqrt{ \Delta} \right]\right).
\end{aligned}
\end{equation}
We can also lower bound the l.h.s. of \eqref{eq:convex_eq1} and \eqref{eq:convex_eq2_epsl0} with 
\begin{equation}
\small
\begin{aligned}
\small
&\E\left[ \sum_{t=1}^T \eta_t \delta_t\right] 
\geq \E\left[ \eta_T \Delta \right] 
\geq \frac{ \left( \E \left[  \Delta^{\nicefrac{1}{2}-\epsilon}\right] \right)^\frac{1}{\nicefrac{1}{2}-\epsilon} }{\left( \E\left[  (\frac{1}{\eta_T})^{\frac{\nicefrac{1}{2}-\epsilon}{\nicefrac{1}{2}+\epsilon}}\right] \right)^{\frac{\nicefrac{1}{2}+\epsilon}{\nicefrac{1}{2}-\epsilon}}},
\end{aligned}
\end{equation}
where the second inequality is due to H\"{o}lder's inequality, i.e. $\E [B^p] \geq \dfrac{\E[AB]^p}{\E[A^q]^{\nicefrac{p}{q}}}$, with $\frac{1}{p} = \frac{1}{2}-\epsilon$, $\frac{1}{q} = \frac{1}{2}+\epsilon$, $A= (\frac{1}{\eta_T})^{\frac{1}{p}}$, and $B = \left[\eta_T \Delta\right]^{\frac{1}{p}}$.
We also have
{\small
\begin{align}
&\frac{1}{\eta_T}
= \frac{1}{\alpha}\left(\beta+\sum_{t=1}^{T-1} \|\bg(\bx_t,\xi_t)\|^2\right)^{\nicefrac{1}{2}+\epsilon}\\
& \leq \frac{1}{\alpha}\left(\beta+2\sum_{t=1}^{T-1} \left(\|\nabla f(\bx_t)-\bg(\bx_t,\xi_t)\|^2+ \|\nabla f(\bx_t)\|^2\right)\right)^{\nicefrac{1}{2}+\epsilon} \\
& \leq \frac{1}{\alpha}\left(\beta+2\sum_{t=1}^{T-1} \big(\|\nabla f(\bx_t)-\bg(\bx_t,\xi_t)\|^2 + 2M \delta_t\big) \right)^{\nicefrac{1}{2}+\epsilon},
\end{align}
}
where in the first inequality we used the elementary inequality $\|\bx+\by\|^2 \leq 2\|\bx\|^2 +2 \|\by\|^2$ and Lemma~\ref{lemma:smooth} in the second one.

Define
{\small
\[
\gamma
=\frac{1}{\alpha(1-\frac{4\alpha M}{\beta^{\nicefrac{1}{2}+\epsilon}})} \big(\|\bx^\star-\bx_1\|^2  + \frac{4\alpha^2}{\beta^{1+2\epsilon}} (1+\ln T) \sigma^2 \big) + K,
\]
}
where $K$ will be defined in the following for the case $\epsilon=0$ and $\epsilon>0$.

When $\epsilon >0$, we have
{\small
\begin{align}
& \frac{1}{\gamma^\frac{\nicefrac{1}{2}-\epsilon}{\nicefrac{1}{2}+\epsilon}}\left(\E \left[\Delta^{\nicefrac{1}{2}-\epsilon}\right] \right)^\frac{1}{\nicefrac{1}{2}+\epsilon} 
\leq  \alpha^\frac{\nicefrac{1}{2}-\epsilon}{\nicefrac{1}{2}+\epsilon} \E \left[ \left(\frac{1}{\eta_T}\right)^{\frac{\nicefrac{1}{2}-\epsilon}{\nicefrac{1}{2}+\epsilon}}\right] \\
& \leq \E \Bigg[ \big( \beta+2\sum_{t=1}^{T-1} (\|\nabla f(\bx_t)-\bg(\bx_t,\xi_t)\|^2  + 2M \delta_t) \big)^{\nicefrac{1}{2}-\epsilon}  \Bigg] \\
& \leq \E  \left[ \left(\beta+2\sum_{t=1}^{T-1} \|\nabla f(\bx_t)-\bg(\bx_t,\xi_t)\|^2\right)^{\nicefrac{1}{2}-\epsilon} \right] \\
& \quad + \E \left[ \left(4M \sum_{t=1}^{T-1} \delta_t \right)^{\nicefrac{1}{2}-\epsilon} \right]  \\
& \leq \left(\beta+2(T-1) \sigma^2 \right)^{\nicefrac{1}{2}-\epsilon} +  (4M)^{\nicefrac{1}{2}-\epsilon}\E \left[ \Delta^{\nicefrac{1}{2}-\epsilon} \right], \label{eq:convex_eq3}
\end{align}
}
where in the third inequality we used Lemma~\ref{lemma: exponential} and we define $K=\frac{\frac{\alpha^2}{2\epsilon \beta^{2\epsilon}}}{\alpha(1-\frac{4\alpha M}{\beta^{\nicefrac{1}{2}+\epsilon}})}$.
Proceeding in the same way, for the case $\epsilon=0$ we get
{\small
\begin{align}
\left( \E \left[  \sqrt{\Delta}\right]\right)^2 
& \leq \left(A +B\E \left[  \sqrt{\Delta}\right] \right) \\ 
& \quad \times \left( C + D \ln \left(A + B \E \left[  \sqrt{\Delta}\right] \right) \right),
\end{align}
}
where 
$A = \sqrt{\beta + 2T \sigma^2}$, 
$B = 2\sqrt{M}$, 
$D = \frac{\alpha}{1-\frac{4\alpha M}{\sqrt{\beta}}}$ and 
$C = \frac{\beta \| \bx_1 -\bx^{\star} \|^2 + 4\alpha^2 (1+\ln T)\sigma^2}{2\alpha \beta (1-\frac{4\alpha M }{\sqrt{\beta}})}$.
Using Lemma~\ref{lemma:logsolvex}, we have that 
{\small
\[
\E \left[  \sqrt{\Delta }\right] 
\leq 32 B^3 D^2 + 2 B C + 8 B^2 D \sqrt{ C} + \frac{A}{B}.
\]
}
We use this upper bound in the logarithmic term, so that for $\epsilon\geq0$, we have \eqref{eq:convex_eq3} again, this time with $K=D \ln(2A + 32 B^4 D^2 + 2 B^2 C + 8 B^3 D \sqrt{C}) = O(\frac{\ln T}{1-\frac{4\alpha M}{\sqrt{\beta}}})$.

Hence, we proceed using Lemma~\ref{lemma:solvex} to have for $\epsilon\geq0$
{\small
\begin{align}
&\E\left[ \Delta^{\nicefrac{1}{2}-\epsilon} \right] \\
&\leq  \max \left(  2^{\frac{\nicefrac{1}{2}+\epsilon}{\nicefrac{1}{2}-\epsilon}}(4M)^{\nicefrac{1}{2}+\epsilon}\gamma  ,  2^{\nicefrac{1}{2}+\epsilon}\gamma^{\nicefrac{1}{2}-\epsilon} \left( \beta+2T\sigma^2  \right)^{\nicefrac{1}{4}-\epsilon^2}  \right).
\end{align}
}
Using Jensen's inequality on the l.h.s. of last inequality concludes the proof.
\end{proof}

\subsection{Adaptive Convergence for Non-Convex Functions}
\label{sec:adapt}

We now prove that the generalized AdaGrad stepsizes in \eqref{eq:eta} allow a faster convergence of the gradients to zero when the noise over the gradients is small.

Given that SGD is not a descent method, we are not aware of any result of convergence with an explicit rate for the last iterate for non-convex functions. Hence, here we will prove a convergence guarantee for the \emph{best iterate} over $T$ iterations rather than for the \emph{last one}.
Note that choosing a random stopping time as in \citet{Ghadimi13} would be equivalent in expectation to choose the best iterate. For simplicity, we choose to state the theorem for the best iterate.
\begin{theorem}
\label{thm:sgd_adaptive}
Assume (\textbf{H1}, \textbf{H3}, \textbf{H4'}). Let the stepsizes set as \eqref{eq:eta}, where $\alpha,\beta>0$, $\epsilon \in(0,\frac{1}{2})$, and $2\alpha M<\beta^{\frac12+\epsilon}$.
Then, the iterates of SGD satisfies the following bound
{\small
\begin{align}
& \E \left[ \min_{1\leq t \leq T} \|\nabla f(\bx_t)\|^{1-2\epsilon}\right]  \\
& \leq \frac{1}{T^{\nicefrac{1}{2}-\epsilon}} \max \left(2^{\frac{\nicefrac{1}{2}+\epsilon}{\nicefrac{1}{2}-\epsilon}}\gamma,
 2^{\nicefrac{1}{2}+\epsilon} \left( \beta+2T\sigma^2 \right)^{\nicefrac{1}{4}-\epsilon^2}\gamma^{\nicefrac{1}{2}-\epsilon} \right),
\end{align}
}
where 
$
\gamma = \begin{cases}
O \left( \frac{1+\alpha^2 \ln T}{\alpha (1-\frac{2\alpha}{\sqrt{\beta}})} \right) & \text{for } \epsilon =0\\
O \left(\frac{1+ \alpha^2 (\frac{1}{\epsilon}+\sigma^2 \ln T)}{\alpha (1-\frac{2\alpha}{\beta^{\nicefrac{1}{2}+\epsilon}})} \right)& \text{for } \epsilon >0 .
\end{cases}
$

\end{theorem}

\paragraph{Remark.}
As in the previous Section, using Markov's inequality it's easy to get that, with probability at least $1-\delta$, 
{\small
\begin{align}
& \min_{1 \leq t \leq T} \| \nabla f(\bx_t) \|^2 \\
& \leq \frac{1}{\delta^{\frac{1}{\nicefrac{1}{2}-\epsilon}} T} \max \left( 2^{\nicefrac{1}{2}+\epsilon} \gamma^{\frac{1}{\nicefrac{1}{2}-\epsilon}}, 
2^{\frac{\nicefrac{1}{2}+\epsilon}{\nicefrac{1}{2}-\epsilon}} (\beta + 2T \sigma^2)^{\nicefrac{1}{2}+\epsilon} \gamma \right).
\end{align}
}

This theorem mirrors Theorem~\ref{thm:convex}, proving again a convergence rate that is adaptive to the noise level. Hence, the same observations on adaptation to the noise level and convergence hold here as well. The main difference w.r.t. Theorem~\ref{thm:convex} is that here we only prove that the gradients are converging to zero rather than the suboptimality gap, because we do not assume convexity.

Note that such bounds were already known with an oracle tuning of the stepsizes, in particular with the knowledge of the variance of the noise, see, e.g., \citet{Ghadimi13}. In fact, the required stepsize in the deterministic case must be constant, while it has to be of the order of $O(\frac{1}{\sigma\sqrt{t}})$ in the stochastic case. However, here we obtain the same behavior automatically, without having to estimate the variance of the noise, thanks to the adaptive stepsizes. This shows for the first time a clear advantage of the global generalized AdaGrad stepsizes over plain SGD.

\begin{proof}[Proof of Theorem~\ref{thm:sgd_adaptive}]
For simplicity, denote by $\Delta:=\sum_{t=1}^T \|\nabla f(\bx_t)\|^2$.

From Lemma~\ref{lemma:basic_lemma}, we have
{\small
\[
\sum_{t=1}^T \E[\eta_t \|\nabla f(\bx_t)\|^2]
\leq f(\bx_1)-f^\star + \frac{M}{2} \E\left[\sum_{t=1}^T \eta_t^2 \|\bg(\bx_t,\xi_t)\|^2\right].
\]
}

Using Lemma~\ref{lemma:bounded_sum_squares}, we can upper bound the expected sum in the r.h.s. of last inequality. When $\epsilon >0$,  we have
{\small
\begin{align}
& \left( \frac{1}{1-\frac{2 \alpha M}{\beta^{\nicefrac12+\epsilon}}} \right) \E\left[\sum_{t=1}^T \eta_t \|\nabla f(\bx_t)\|^2\right] 
\leq f(\bx_1)-f^\star \nonumber \\
&\quad +\frac{\alpha^2 M}{4\epsilon \beta^{2\epsilon}} + \frac{2\alpha^2 \sigma^2 M}{\beta^{1+2\epsilon}} (1+\ln T). \label{eq:sgd_adaptive_eq1}
\end{align}
}

When $\epsilon=0$, we have
{\small
\begin{align}
& \left( \frac{1}{1-\frac{2\alpha M}{\sqrt{\beta}}} \right) \E\left[\sum_{t=1}^T \eta_t \|\nabla f(\bx_t)\|^2\right]  
\leq f(\bx_1)-f^{\star} \nonumber \\
& \quad + M\alpha^2 \ln \left( \sqrt{\beta + 2T\sigma^2 }  + \sqrt{2} \E\left[ \sqrt{ \sum_{t=1}^T \eta_t \|\nabla f(\bx_t)\|^2 } \right] \right)  \nonumber \\
&\quad + \frac{2\alpha M }{\beta} (1+\ln T)\sigma^2. \label{eq:sgd_adaptive_eq2}
\end{align}
}

With similar methods in the proof of Theorem~\ref{thm:convex}, we lower bound the l.h.s. of  both \eqref{eq:sgd_adaptive_eq1} and \eqref{eq:sgd_adaptive_eq2} with 
{\small
\begin{align*}
\E\left[ \sum_{t=1}^T \eta_t \|\nabla f(\bx_t)\|^2\right]
&\geq \E\left[ \eta_T \Delta\right] 
= \E\left[ \eta_T \Delta \right] \\
&\geq \frac{ \left( \E \left[ \Delta^{\nicefrac{1}{2}-\epsilon} \right] \right)^\frac{1}{\nicefrac{1}{2}-\epsilon} }{\left( \E \left[ (\frac{1}{\eta_T})^{\frac{\nicefrac{1}{2}-\epsilon}{\nicefrac{1}{2}+\epsilon}} \right] \right)^{\frac{\nicefrac{1}{2}+\epsilon}{\nicefrac{1}{2}-\epsilon}}}.
\end{align*}
}
We also have 
{\small
\begin{align*}
& \frac{1}{\eta_T} = \frac{1}{\alpha}\left(\beta+\sum_{t=1}^{T-1} \|\bg(\bx_t,\xi_t)\|^2\right)^{\nicefrac{1}{2}+\epsilon} \\
& \leq \frac{1}{\alpha}\left(\beta+2\sum_{t=1}^{T-1} \left(\|\nabla f(\bx_t)-\bg(\bx_t,\xi_t)\|^2+ \|\nabla f(\bx_t)\|^2\right)\right)^{\nicefrac{1}{2}+\epsilon}.
\end{align*}
}

Define 
{\small
\[
\gamma = \frac{1}{\alpha (1-\frac{2\alpha M}{\beta^{\nicefrac{1}{2}+\epsilon}})} \left( f(\bx_1)-f^{\star}+ \frac{2\alpha^2 M }{\beta^{1+2\epsilon}}\sigma^2 \right) + K,
\]
}
where $K$ will be defined separately for the case $\epsilon = 0$ and $\epsilon >0.$

When $\epsilon >0$, we have 
{\small
\begin{align}
& \left( \E \left[ \Delta^{\nicefrac{1}{2}-\epsilon} \right] \right)^\frac{1}{\nicefrac{1}{2}-\epsilon} \\
& \leq  \alpha \gamma \left( \E \left[ (\frac{1}{\eta_T})^{\frac{\nicefrac{1}{2}-\epsilon}{\nicefrac{1}{2}+\epsilon}} \right] \right)^{\frac{\nicefrac{1}{2}+\epsilon}{\nicefrac{1}{2}-\epsilon}} \\
& \leq \gamma \left( \E  \left[ \left(\beta+2\sum_{t=1}^{T-1} \|\nabla f(\bx_t)-\bg(\bx_t,\xi_t)\|^2\right)^{\nicefrac{1}{2}-\epsilon} \right] \right.\\
& \quad \left. + 2\E \left[ \left(\sum_{t=1}^{T-1} \|\nabla f(\bx_t)\|^2 \right)^{\nicefrac{1}{2}-\epsilon} \right] \right)^{\frac{\nicefrac{1}{2}+\epsilon}{\nicefrac{1}{2}-\epsilon}} \\
& \leq \gamma \left( \left(\beta+2T \sigma^2 \right)^{\nicefrac{1}{2}-\epsilon} + 2\E \left[ \Delta^{\nicefrac{1}{2}-\epsilon}\right]  \right)^{\frac{\nicefrac{1}{2}+\epsilon}{\nicefrac{1}{2}-\epsilon}}. \label{eq:adapt_eq3}
\end{align}
}
where in this case we define $K = \frac{\frac{\alpha^ M }{4\epsilon \beta^{2\epsilon}}}{\alpha(1-\frac{2\alpha M}{\beta^{\nicefrac{1}{2}+\epsilon}})}.$
Proceeding in the same way, when $\epsilon=0$, we have 
{\small
\begin{align}
\left( \E \left[ \sqrt{\Delta } \right] \right)^2 
&\leq \left(A + B \E \left[ \sqrt{\Delta } \right]  \right) \\
& \qquad \times \left(  C + D \ln \left(A +B \E \left[ \sqrt{\Delta } \right] \right) \right),
\end{align}
}
where $A = \sqrt{\beta + 2T \sigma^2}$, 
$B = \sqrt{2}$,
$D = \frac{\alpha M}{1-\frac{2\alpha M}{\sqrt{\beta}}}$, 
$C = \frac{\beta (f(\bx_1)-f^{\star}) + 2\alpha (1+\ln T)\sigma^2}{\alpha \beta (1-\frac{2\alpha M }{\sqrt{\beta}})}$. 

Using Lemma~\ref{lemma:logsolvex}, we have that  
{\small
\begin{align}
\E \left[ \sqrt{\Delta } \right] 
\leq 32 B^3 D^2 + 2 B C + 8 B^2 D \sqrt{ C} + \frac{A}{B}.
\end{align}
}
Similar with Theorem~\ref{thm:convex}, we use this upper bound in the logarithmic term so that for $\epsilon \geq 0$, we have \eqref{eq:adapt_eq3} again, this time with $K = D \ln (2A + 32 B^4 D^2 + 2 B^2 C +8 B^3 D \sqrt{C}) = O(\frac{\ln T}{1-\frac{2\alpha M }{\beta}}).$

Hence, we proceed using Lemma~\ref{lemma:solvex} to have for $\epsilon \geq 0$
{\small
\begin{align}
\E&\left[ \Delta^{\nicefrac{1}{2}-\epsilon} \right] \\
& \leq \max \left(2^{\frac{\nicefrac{1}{2}+\epsilon}{\nicefrac{1}{2}-\epsilon}} \gamma,
 2^{\nicefrac{1}{2}+\epsilon} \left( \beta+2T \sigma^2 \right)^{\nicefrac{1}{4}-\epsilon^2}\gamma^{\nicefrac{1}{2}-\epsilon} \right).
\end{align}
}
Lower bounding $\E \left[ \Delta^{\nicefrac{1}{2}-\epsilon} \right]$ by $T^{\nicefrac{1}{2}-\epsilon} \E \left[\min_{1\leq t \leq T} \|\nabla f(\bx_t)\|^{1-2\epsilon}\right]$,  we have the stated bound. 
\end{proof}
\section{DISCUSSION AND FUTURE WORK}

We have presented an analysis of adaptive stepsizes based on the generalized AdaGrad stepsizes for stochastic gradient descent, with convex and non-convex functions.
In the convex setting, our result shows an adaptive convergence rate, also overcoming the limitations of previous results. In the non-convex setting, we show almost sure convergence and adaptive convergence rates. Moreover, we show for the first time sufficient condition for a convergence guarantee for non-convex functions for a minor variation of AdaGrad.

We believe these results have twofold importance. First, we go in the direction of closing the gap between theory and practice for widely used optimization algorithms. Second, our adaptive rates provide a possible explanation for the empirical success of these kinds of algorithms in practical machine learning applications. 

One of the limitations of the current analysis is the fact that our analysis implies high probability bounds that depends polynomially on $\frac{1}{\delta}$, due to the application of Markov's inequality. It would be better to prove bounds that depend on $\ln (\frac{1}{\delta})$, as they are possible for SGD under conditions \eqref{eq:conditions_stepsize}.  However, the generalized AdaGrad updates do not satisfy the conditions \eqref{eq:conditions_stepsize} and the analysis is not straightforward. Our future work will focus on shedding light on this issue.

In the future, we would also like to understand if the conditions we impose can be weakened. For example, the almost sure convergence requires a bounded support noise, that, while it might be verified in many practical scenarios, still seems unsatisfying from a theoretical point of view. Moreover, we would like to adapt the recent approaches for parameter-free online optimization~\citep{OrabonaP16b,CutkoskyO18} to the non-convex setting.

\subsubsection*{Acknowledgements}
The authors thank D\'avid P\'al for the comments and discussions and L\'eon Bottou for the comments on prior work. This material is based upon work supported by the National Science Foundation under grant no. 1740762 ``Collaborative Research: TRIPODS Institute for Optimization and Learning'' and by a Google Research Award.

\bibliographystyle{plainnat}
\bibliography{../learning}

\newpage
\onecolumn
\appendix
\section{Appendix}

Here, we report the proofs missing from the main text.

\subsection{Details of Example~\ref{ex:ninety}}

Consider the function $f(x) = \frac{1}{2} x^2$. The gradient in $t$-th iteration is $\nabla f(x_t) = x_t$. 
Let the stochastic gradient be defined as $\bg_t = \nabla f(x_t) + \xi_t$, where $P(\xi_t = \sigma_t) =\frac{7}{15}$, $ P(\xi_t = -\frac{3}{2} \sigma_t) = \frac{1}{5}$
and $ P(\xi_t = -\frac{1}{2} \sigma_t) = \frac{1}{3}$. 

Let $A \triangleq \sum_{i=1}^{t-1} g_i^2+\beta$. Then
\begin{equation*}
\langle \E_t \eta_{t+1} \bg_t , \nabla f(x_t) \rangle 
= \alpha \left[ \frac{7}{15}\frac{(x_t +\sigma_t) x_t}{[A+(x_t+\sigma_t)^2]^{\frac{1}{2}+\epsilon}} + \frac{1}{5}\frac{(x_t -\frac{3}{2}\sigma_t) x_t}{[A+(x_t-\frac{3}{2}\sigma_t)^2]^{\frac{1}{2}+\epsilon}}+\frac{1}{3}\frac{(x_t -\frac{1}{2}\sigma_t) x_t}{[A+(x_t-\frac{1}{2}\sigma_t)^2]^{\frac{1}{2}+\epsilon}} \right].
\end{equation*}
This expression can be negative, for example, setting $x_t=1$, $\sigma_t = 10$, $A=10$, $\epsilon=0$ or $\epsilon = 0.1$. 

\subsection{Proof of Lemma~\ref{lemma:sum_bounded}}

\begin{lemma}
\label{lemma:sum_integral_bounds}
Let $a_i\geq0, \cdots, T$ and $f:[0,+\infty)\rightarrow [0, +\infty)$ nonincreasing function.
Then
\begin{align*}
\sum_{t=1}^T a_t f\left(a_0+\sum_{i=1}^{t} a_i\right) 
&\leq \int_{a_0}^{\sum_{t=0}^T a_t} f(x) dx.
\end{align*}
\end{lemma}

\begin{proof}
Denote by $s_t=\sum_{i=0}^{t} a_i$.
\begin{align*}
a_i f(s_i) 
=  \int_{s_{i-1}}^{s_i} f(s_i) d x 
\leq \int_{s_{i-1}}^{s_i} f(x) d x.
\end{align*}
Summing over $i=1, \cdots, T$, we have the stated bound.
\end{proof}

\begin{proof}[Proof of Lemma~\ref{lemma:sum_bounded}]
The proof is immediate from Lemma~\ref{lemma:sum_integral_bounds}.
\end{proof}

\subsection{Proofs of Section~\ref{sec:convex}}

\begin{proof}[Proof of Lemma~\ref{lemma:smooth}]
From \eqref{eq:smooth2}, for any $\bx,\by \in\R^d$, we have
\[
f(\bx+\by)\leq f(\bx)+\langle \nabla f(\bx), \by\rangle+ \frac{M}{2}\|\by\|^2.
\]
Take $\by=-\frac{1}{M} \nabla f(\bx)$, to have
\[
f(\bx+\by)\leq f(\bx)+\left(\frac{1}{2M}-\frac{1}{M}\right)\| \nabla f(\bx)\|^2.
\]
Hence,
\[
\|\nabla f(\bx) \|^2
\leq 2M(f(\bx)-f(\bx+\by))
\leq 2M(f(\bx)-\min_{\bu} f(\bu)). \qedhere
\]
\end{proof}
\begin{proof}[Proof of Lemma~\ref{lemma:solvex}]
If $A \leq Bx$, then $x \leq C(2Bx)^{\frac{1}{2}+\epsilon}$, so $x \leq \left[ C (2B)^{\frac{1}{2}+\epsilon} \right]^{\frac{1}{1/2-\epsilon}}$. 
And if $A > Bx$, then $x < C(2A)^{\frac{1}{2}+\epsilon}$. 
Taking the maximum of the two cases, we have the stated bound.
\end{proof}

\begin{proof}[Proof of Lemma~\ref{lemma:logsolvex}]
Assume that $B x > A$. We have that
\[
x^2 
\leq (A + B x) (C + D \ln (A+B x))
< 2 B x (C + D \ln (2B x))
< 2 B x (C + 2D\sqrt{2Bx}), 
\]
that is
\[
x 
< 2 B C + 4BD\sqrt{2Bx}.
\]
We can solve this inequality, to obtain
\[
x < 32 B^3 D^2 + 2 B C + 8B^2 D \sqrt{C}.
\]
On the other hand, if $B x\leq A$, we have $x\leq \frac{A}{B}$.
Taking the sum of these two case, we have the stated bound.
\end{proof}

\begin{proof}[Proof of Lemma~\ref{lemma: exponential}]
Let $f(x) = (x+y)^p-x^p-y^p$. We can see that $f'(x) = p(x+y)^{p-1}-px^{p-1} \leq 0$ when $x,y \geq 0$. So $f(x) \leq f(0) = 0$. The inequality holds. 
\end{proof}

\begin{lemma}
\label{lemma: bound log}
If $x>0$, $\alpha >0 $, then $\ln(x) \leq \alpha(x^{\frac{1}{\alpha}}-1)$. 
\end{lemma}
\begin{proof}[Proof of Lemma~\ref{lemma: bound log}]
Let $f(x) = \ln(x)-\alpha x^{\frac{1}{\alpha}}+\alpha$. $f'(x)= \frac{1}{x}-x^{\frac{1}{\alpha}-1}$ is positive when $0<x<1$, $f'(1)=0$ and $f'(x) <0$ when $x>1$. So 
$f(x) \leq f(1)=0.$ The inequality holds. 
\end{proof}
\begin{proof}[Proof of Lemma~\ref{lemma:bounded_sum_squares}]

Using the assumption on the noise, we have
\begin{align*}
\exp&\left(\frac{\E\left[\max_{1 \leq  i\leq T} \|\nabla f(\bx_i) -\bg(\bx_i,\xi_i)\|^2\right]}{\sigma^2}\right) 
\leq \E\left[\exp\left(\frac{\max_{1 \leq  i\leq T} \|\nabla f(\bx_i) -\bg(\bx_i,\xi_i)\|^2}{\sigma^2}\right)\right] \\
&= \E\left[\max_{1 \leq  i\leq T} \exp\left(\frac{\|\nabla f(\bx_i) -\bg(\bx_i,\xi_i)\|^2}{\sigma^2}\right)\right]
\leq \sum_{i=1}^T \E\left[\exp\left(\frac{\|\nabla f(\bx_i) -\bg(\bx_i,\xi_i)\|^2}{\sigma^2}\right)\right] \\
&= \sum_{i=1}^T \E\left[\E_i\left[\exp\left(\frac{\|\nabla f(\bx_i) -\bg(\bx_i,\xi_i)\|^2}{\sigma^2}\right)\right]\right] 
\leq T e,
\end{align*}
that implies
\begin{equation}
\label{eq:proof_lemma:bounded_sum_squares_eq1}
\E\left[\max_{1 \leq  i\leq T} \|\nabla f(\bx_i) -\bg(\bx_i,\xi_i)\|^2\right] \leq \sigma^2 (1+\ln T).
\end{equation}
Hence, when $\epsilon >0$, we have
\begin{align*}
\E\left[\sum_{t=1}^T \eta_t^2 \|\bg(\bx_t,\xi_t)\|^2\right]
&= \E\left[\sum_{t=1}^T \eta_{t+1}^2 \|\bg(\bx_t,\xi_t)\|^2 + \sum_{t=1}^T \|\bg(\bx_t,\xi_t)\|^2 (\eta_t^2 -\eta^2_{t+1})\right]\\
&= \E\left[\sum_{t=1}^T \eta_{t+1}^2 \|\bg(\bx_t,\xi_t)\|^2 + \sum_{t=1}^T \|\bg(\bx_t,\xi_t)\|^2 (\eta_t +\eta_{t+1})(\eta_t -\eta_{t+1})\right]\\
&\leq \E\left[\sum_{t=1}^T \eta_{t+1}^2 \|\bg(\bx_t,\xi_t)\|^2 + \sum_{t=1}^T 2 \eta_t \|\bg(\bx_t,\xi_t)\|^2 (\eta_t -\eta_{t+1})\right]\\
&\leq \frac{\alpha^2}{2\epsilon \beta^{2\epsilon}} + 2 \eta_1 \E\left[\max_{1 \leq t \leq T} \eta_t \|\bg(\bx_t,\xi_t)\|^2\right] \\
&\leq \frac{\alpha^2}{2\epsilon \beta^{2\epsilon}} + 4 \eta_1 \E\left[\max_{1 \leq t \leq T} \eta_t \left(\|\bg(\bx_t,\xi_t)-\nabla f(\bx_t)\|^2+\|\nabla f(\bx_t)\|^2\right)\right] \\
&\leq \frac{\alpha^2}{2\epsilon \beta^{2\epsilon}} + 4 \eta^2_1 (1+\ln T) \sigma^2 + 4 \eta_1 \E\left[\sum_{t=1}^T \eta_t \|\nabla f(\bx_t)\|^2\right] \\
&= \frac{\alpha^2}{2\epsilon \beta^{2\epsilon}} + \frac{4\alpha^2}{\beta^{1+2\epsilon}} (1+\ln T) \sigma^2 + \frac{4\alpha}{\beta^{\frac{1}{2}+\epsilon}} \E\left[\sum_{t=1}^T \eta_t \|\nabla f(\bx_t)\|^2\right], 
\end{align*}
where in second inequality we used Lemma~\ref{lemma:sum_bounded} and in fourth one we used \eqref{eq:proof_lemma:bounded_sum_squares_eq1}. Note that the analysis after the second inequality also holds when $\epsilon=0$. 

And when $\epsilon = 0$, we have 
\begin{align*}
& \E\left[ \sum_{t=1}^T \eta_{t+1}^2 \|\bg(\bx_t,\xi_t)\|^2 \right] 
=  \E \left[ \sum_{t=1}^T \frac{\alpha^2 \|\bg(\bx_t,\xi_t)\|^2 }{(\beta+\sum_{i=1}^t \| \bg( \bx_i,\xi_t) \|^2)} \right] \\
& \leq 2 \alpha^2 \E  \left[ \ln \left( \sqrt{\beta + \sum_{t=1}^T \|\bg(\bx_t,\xi_t)\|^2}  \right) \right] \\
& \leq 2 \alpha^2 \E \left[ \ln \left( \sqrt{\beta + 2 \sum_{t=1}^T \|\bg(\bx_t,\xi_t)-\nabla f(\bx_t)\|^2} +\sqrt{2 \sum_{t=1}^T \|\nabla f(\bx_t)\|^2}  \right)\right] \\
& \leq 2 \alpha^2  \ln \left( \sqrt{\beta + 2T\sigma^2 }+\sqrt{2} \E \left[ \sqrt{ \sum_{t=1}^T \|\nabla f(\bx_t)\|^2}  \right] \right) \\
\end{align*}
where in first inequality we used Lemma~\ref{lemma: bound log} and in the third one we used Jensen's inequality. Putting things together, we have 
\begin{align*}
& \E\left[\sum_{t=1}^T \eta_t^2 \|\bg(\bx_t,\xi_t)\|^2\right]
= \E\left[\sum_{t=1}^T \eta_{t+1}^2 \|\bg(\bx_t,\xi_t)\|^2 + \sum_{t=1}^T \|\bg(\bx_t,\xi_t)\|^2 (\eta_t^2 -\eta^2_{t+1})\right]\\
& \leq 2 \alpha^2  \ln \left( \sqrt{\beta + 2T\sigma^2 }+\sqrt{2} \E \left[ \sqrt{ \sum_{t=1}^T \|\nabla f(\bx_t)\|^2}  \right] \right) + \frac{4\alpha^2}{\beta} (1+\ln T) \sigma^2 + \frac{4\alpha}{\beta^{\frac{1}{2}}} \E\left[\sum_{t=1}^T \eta_t \|\nabla f(\bx_t)\|^2\right]
\end{align*}

\end{proof}

\subsection{Proofs of Section~\ref{sec:almost_sure}}

\begin{proof}[Proof of Lemma~\ref{lemma:basic_lemma}]
From \eqref{eq:smooth2}, we have
\begin{align}
f(\bx_{t+1}) 
&\leq f(\bx_t) + \langle \nabla f(\bx_t), \bx_{t+1}-\bx_t\rangle + \frac{M}{2}\|\bx_{t+1}-\bx_t\|^2 \\
&= f(\bx_t) + \langle \nabla f(\bx_t), \bta_t(\nabla f(\bx_t) - \bg(\bx_t,\xi_t))\rangle - \langle \nabla f(\bx_t), \bta_t \nabla f(\bx_t)\rangle + \frac{M}{2}\|\bta_t\bg(\bx_t,\xi_t)\|^2.
\end{align}
Taking the conditional expectation with respect to $\xi_1, \cdots, \xi_{t-1}$, we have that
\[
E_t[\langle \nabla f(\bx_t), \bta_t(\nabla f(\bx_t) - \bg(\bx_t,\xi_t)) \rangle]
=  \langle \nabla f(\bx_t), \bta_t \nabla f(\bx_t) - \bta_t \E_t[\bg(\bx_t,\xi_t)] \rangle
= 0.
\]
Hence, from the law of total expectation, we have
\[
\E\left[\langle \nabla f(\bx_t), \bta_t \nabla f(\bx_t)\rangle\right]
\leq \E\left[f(\bx_t)- f(\bx_{t+1}) + \frac{M}{2}\|\bta_t \bg(\bx_t,\xi_t)\|^2\right].
\]
Summing over $t=1$ to $T$ and lower bounding $f(\bx_{T+1})$ with $f^\star$, we have the stated bound.
\end{proof}

\begin{proof}[Proof of Lemma~\ref{lemma:remove_liminf}]
Since the series $\sum_{t=1}^{\infty}a_t$ diverges, given that $\sum_{t=1}^{\infty}a_t b_t$ converges, we necessarily have $\liminf_{t \rightarrow \infty}b_t = 0$. So there exists a subsequence $\{ b_{i(t)} \}$ of $\{ b_t \}$ such that $ \lim_{t \to \infty} b_{i(t)} =0 .$

Let us proceed by contradiction and assume that there exists some $\alpha> 0 $ and some other subsequence $ \{  b_{m(t)}\}$ of $\{ b_t \}$ such that $ b_{m(t)} \geq \alpha$ for all $t$. In this case, we can construct a third subsequence $\{ b_{j(t)}\}$ of $\{ b_t \}$ where the subindices $j(t)$ are chosen in the following way: 
\begin{equation}
\label{eq: indice_1}
j(0) = \min \{ l \geq 0:  b_l \geq \alpha \}
\end{equation}

and, given $j(2t)$,

\begin{equation}
\label{eq: indice_2}
j(2t+1) = \min \{ l \geq j(2t) : b_l \leq \frac{1}{2} \alpha \},
\end{equation}
\begin{equation}
\label{eq: indice_3}
j(2t+2) = \min \{ l \geq j(2t+1): b_l \leq \frac{1}{2} \alpha \}.
\end{equation}

Note that the existence of $\{ b_{i(t)} \}$ and $\{ b_{m(t)} \}$ guarantees that $j(t)$ is well defined. Also by $\eqref{eq: indice_2}$ and $\eqref{eq: indice_3}$
\[
b_l \leq \frac{\alpha}{2} \text{for } j(2t) \leq l \leq j(2t+1)-1.
\]
Then, denoting $\phi_t  = \sum_{l=2t}^{j(2t+1)-1} a_l $, we have 
\[
\infty > \sum_{t=1}^{\infty} a_t b_t \geq \sum_{t=1}^{\infty} \sum_{l=2t}^{j(2t+1)-1} a_l b_l \leq \frac{\alpha}{2} \sum_{t=1}^{\infty} \phi_t.
\]
Therefore, we have $\lim_{t \to \infty} \phi_t = 0.$

On the other hand, by $\eqref{eq: indice_2}$ and $\eqref{eq: indice_3}$, we have $b_{j(2t)} \geq \alpha$, $b_{j(2t+1)} \leq \frac{1}{\alpha}$, so that 
\[
\frac{\alpha}{2} \leq b_{j(2t)}-b_{j(2t+1)} = \sum_{l=j(2t)}^{j(2t+1)-1} (b_l - b_{l+1}) \leq \sum_{l=j(2t)}^{j(2t+1)-1} K a_l = K \phi_t. 
\]
So $\phi_t \geq \frac{\alpha}{2K}$, which is in contradiction with $\lim_{t \to \infty} \phi_t = 0.$
Therefore, $b_t$ goes to zero. 

\end{proof}

\begin{proof}[Proof of Theorem~\ref{thm:convergence_adagrad}]
We proceed similarly to the proof of Theorem~\ref{thm:convergence_sgd}, to get 
\[
\E\left[\sum_{t=1}^\infty \langle \nabla f(\bx_t), \bta_t \nabla f(\bx_t) \rangle\right] 
\leq f(\bx_1)- f(\bx^\star)+ \frac{M}{2}\E\left[\sum_{t=1}^\infty \|\bta_t \bg(\bx_t,\xi_t)\|_{2}^2\right].
\]
Observe that
\[
\sum_{t=1}^{\infty}  \|\bta_t \bg(\bx_t,\xi_t)\|^2 
= \sum_{t=1}^{\infty} \sum_{i=1}^d \eta_{t,i}^2 \bg(\bx_t,\xi_t)_i^2 
= \sum_{i=1}^d \sum_{t=1}^{\infty} \eta_{t,i}^2 \bg(\bx_t,\xi_t)_i^2
 < \infty,
\]
where the last inequality comes from the same reasoning in \eqref{eq:convergence_sgd_eq1}.
Hence, we have
\[
\E\left[\sum_{t=1}^{\infty} \langle \nabla f(\bx_t), \bta_t \nabla f(\bx_t)\rangle\right] < \infty.
\]
Hence, with probability 1, we have
\[
\sum_{t=1}^{\infty} \langle \nabla f(\bx_t), \bta_t \nabla f(\bx_t)\rangle 
= \sum_{t=1}^{\infty} \sum_{j=1}^d \eta_{t,j} \nabla f(\bx_t)_{j}^2 
= \sum_{j=1}^d \sum_{t=1}^{\infty} \eta_{t,j} \nabla f(\bx_t)_{j}^2
< \infty.
\]
and, for any $j = 1,\cdots,d$, 
\[
\sum_{t=1}^{\infty} \eta_{t,j} (\nabla f(\bx_t))_{j}^2 < \infty.
\]
Now, observe that the Lipschitzness of $f$ and the bounded support of the noise on the gradients gives
\begin{align*}
\sum_{t=1}^{\infty} \eta_{t,j}
= \sum_{t=1}^ {\infty} \frac{\alpha}{(\beta+\sum_{i=1}^{t-1} (g(\bx_i,\xi_i)_j)^2)^{\nicefrac12+\epsilon}}
\geq  \sum_{t=1}^ {\infty} \frac{\alpha}{(\beta+2(t-1)(L^2+S^2))^{\nicefrac12+\epsilon}}
= \infty.
\end{align*}
Using the fact the $f$ is $L$-Lipschitz and $M$-smooth, we also have
\begin{align*}
&\left| ((\nabla f(\bx_{t+1}))_j)^2 - ((\nabla f(\bx_t))_j)^2\right| 
 = ( (\nabla f(\bx_{t+1}))_j+ (\nabla f(\bx_t))_j) \cdot \left| (\nabla f(\bx_{t+1}))_j- (\nabla f(\bx_t))_j \right| \\
 &\quad\leq 2L M \|\bx_{t+1}-\bx_t\|
 = 2LM \|\bta_t \bg(\bx_t,\xi_t)\| 
 \leq 2LM(L+S) \eta_t.
\end{align*}
Hence, we case use Lemma~\ref{lemma:remove_liminf} to obtain 
\[
\lim_{t \to \infty} ((\nabla f(\bx_t))_j)^2 = 0.
\]

For the second statement, observe that, with probability 1,
\begin{align*}
\sum_{t=1}^\infty ((\nabla f(\bx_t))_j)^2 t^{\nicefrac12-\epsilon} \frac{\alpha}{t(2L^2+2S^2+\beta)^{\nicefrac12+\epsilon}} 
\leq \sum_{t=1}^\infty \eta_{t,j} (\nabla f(\bx_t))_j)^2 <\infty.
\end{align*}
Hence, noting that $\sum_{t=1}^\infty \frac{1}{t} =\infty$, we have that $\lim\inf_{t\rightarrow \infty} ((\nabla f(\bx_t))_j)^2 t^{\nicefrac12-\epsilon}=0$.
\end{proof}

\end{document}